\documentclass[letterpaper, 10 pt, conference]{ieeeconf}

\IEEEoverridecommandlockouts
\usepackage{amsmath,amssymb}

\usepackage{amsthm}
\usepackage{graphicx}
\usepackage{cite}
\newtheorem{remark}{Remark}
\newtheorem{lemma}{Lemma}
\usepackage{algorithmic}
\usepackage{algorithm}
\usepackage{multirow}
\usepackage{verbatim}
\DeclareMathOperator*{\argmax}{arg\,max}
\DeclareMathOperator*{\argmin}{arg\,min}

\graphicspath{ {./figures/} }

\title{\LARGE \bf Stress-Sharing for Decentralized Fault Repair in Modular Spacecraft}

\author{Sidhdharth D. Sikka$^{1,2}$, Yue Shen$^{2}$, and Shaoshuai Mou$^{1}$%
\thanks{This work has been submitted to the IEEE for possible publication. Copyright may be transferred without notice, after which this version may no longer be accessible.}%
\thanks{$^{1}$School of Aeronautics and Astronautics, Purdue University, West Lafayette, IN 47907 USA (e-mail: sikkas@purdue.edu; mous@purdue.edu).}%
\thanks{$^{2}$Manifold Research Group (e-mail: yzs.shen@gmail.com).}%
}

\begin{document}

\maketitle
\thispagestyle{empty}
\pagestyle{empty}

\begin{abstract}
Structural damage in modular spacecraft can disrupt mechanical and communication connectivity, reducing system capability. Existing approaches rely on redundancy or preplanned reconfiguration and do not enable autonomous repair under local information and physical constraints. We model the spacecraft as a lattice-constrained graph and introduce a fully decentralized, asynchronous stress-sharing repair policy inspired by biological wound healing: local distress signals guide surviving modules toward damaged regions to close fragmented gaps, after which each displaced module locally retraces its own motions to recover the pre-damage shape, using only local information and no absolute position sensing. We evaluate the policy in PyBullet rigid-body simulation across structures of up to $160$ modules, three fault densities ($10$, $20$, $30\%$), and random and localized damage. The policy consolidates the surviving modules into a single connected body: even in the most severe case tested, where $30\%$ of modules fail at random, it gathers roughly $80\%$ or more of the surviving modules into one connected component, and this fraction improves with assembly size, making the approach well suited as a swarm-scale repair policy for large modular spacecraft.
\end{abstract}

\section{Introduction}

Spacecraft on long-duration missions must tolerate failures with no opportunity for manual repair, yet traditional mitigation relies on redundancy and conservative design margins that add mass and cost while remaining vulnerable to unanticipated and cascading faults. In response, the field is shifting toward modular, reconfigurable architectures that can lower lifecycle cost and respond to damage by physically reorganizing their surviving components~\cite{zhang2023modularity}, reflecting the autonomy that on-orbit servicing and in-space assembly will require~\cite{arney2021osam}. However, most distributed-spacecraft and swarm-assembly approaches~\cite{liao2025assembly,nanjangud2024towards} still rely on attached manipulator arms or dedicated free-flying service robots to reposition modules~\cite{xue2021review}, lacking the intrinsic, highly distributed structural plasticity needed for rapid fault recovery in degraded environments.

In this paper, we consider the problem of \emph{damage-responsive reconfiguration} in modular spacecraft. Following a localized failure, a subset of modules becomes inactive and the remaining structure must restore operational connectivity using only local sensing and physically admissible motions (Fig.~\ref{fig:reconfig}). Unlike global shape planning, this setting imposes three coupled constraints: (i) motion is limited to lattice-admissible pivot operations, (ii) connectivity must be preserved throughout execution, and (iii) agents operate without centralized coordination or global state information.

\begin{figure}[t]
\centering
\includegraphics[width=0.8\columnwidth]{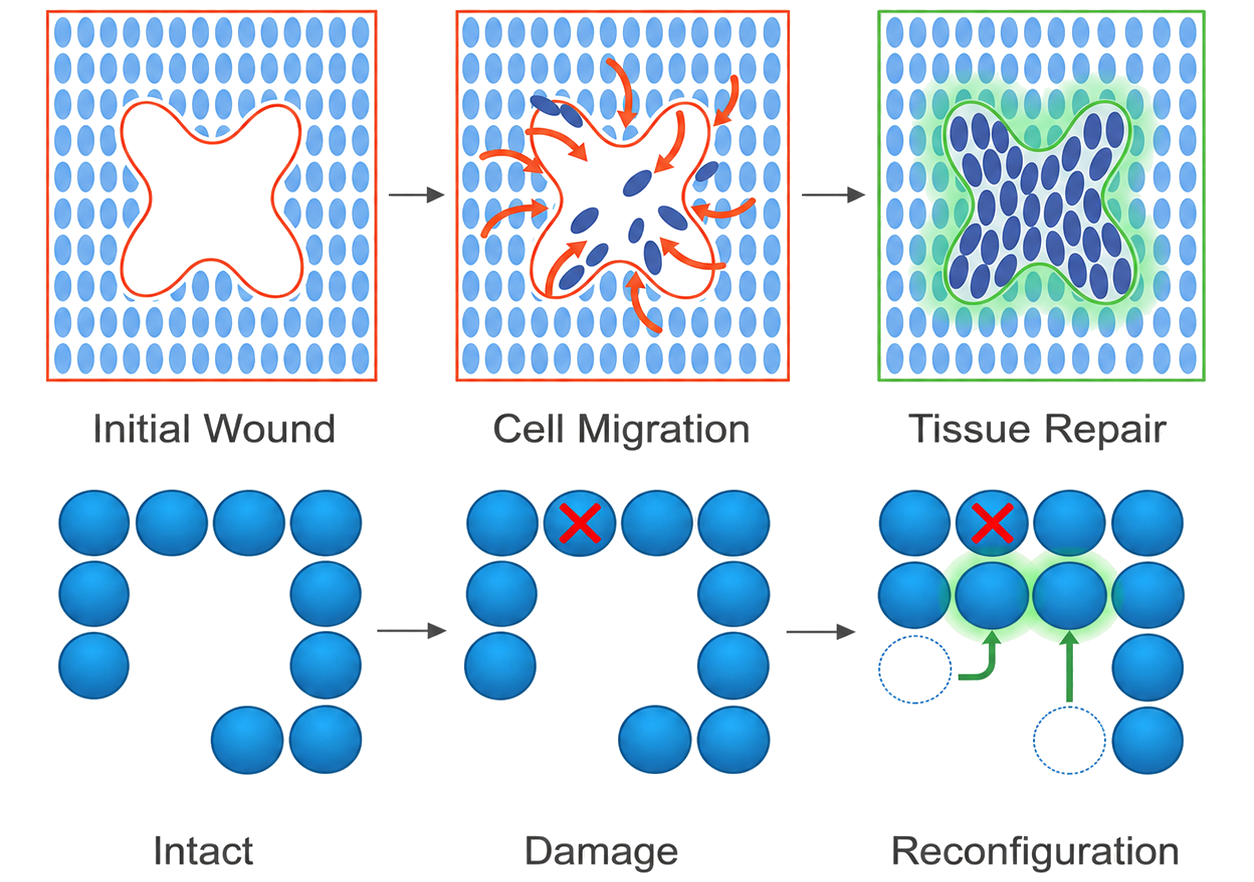}
\caption{Biological wound healing (top) inspires damage-responsive reconfiguration in a modular robotic system (bottom): neighboring modules migrate toward a localized failure to close the resulting gap.}
\label{fig:reconfig}
\end{figure}

\subsection{Related Work}

\paragraph{Modular reconfiguration systems}
A broad set of literature has studied hardware and algorithms for modular robotic reconfiguration. Lattice-based systems such as M-Blocks demonstrate dynamic pivoting through inertial actuation~\cite{romanishin2015mblocks}, while ElectroVoxel achieves pivot-based reconfiguration in microgravity~\cite{shay2022electrovoxel}. At the algorithmic level, universal reconfiguration results show that arbitrary shapes can be achieved with a constant number of helper modules~\cite{akitaya2019universal}. More generally, modular reconfiguration has been studied through decentralized control, distributed motion planning, cellular automata, and lattice-based transformation rules~\cite{naz2016distributed,zhu2017distributed,lengiewicz2019efficient}.

\paragraph{Distributed self-repair paradigms}
Classical distributed self-repair strategies often rely on paradigms that are poorly matched to spacecraft constraints. Scale-independent morphallaxis, as developed for robotic collectives and related regenerative swarm robotics, rebuilds a proportionately scaled version of the structure after damage~\cite{rubenstein2009biologically}. Other approaches, including directed-growth and hole-routing style methods for modular robots~\cite{stoy2004self,naz2016distributed}, attempt to migrate empty space or restructure the body so that missing components can be replaced.
\paragraph{Connectivity maintenance and network repair}
A related literature addresses connectivity preservation and restoration in networked multi-agent systems. In wireless sensor and actor networks, RECRA and related localized motion-based restoration algorithms relocate non-critical nodes to repair broken communication graphs using local information~\cite{imran2012localized,zhang2018autonomous}. In multi-robot systems, decentralized connectivity maintenance has been studied through graph-theoretic and control-theoretic methods, including local-global approaches, bounded-input connectivity preservation, and distributed connectivity control~\cite{khateri2020decentralized,tardioli2010enforcing,cornejo2009keeping}.

\paragraph{Biologically inspired and regenerative systems}
Biological inspiration has also played a central role in self-reconfiguration and repair. In wound healing, cells migrate toward an injury site under local chemical and mechanical cues and coordinate through signaling pathways to close the gap, all without centralized control, and this analogy directly motivates our approach. Hormone-based coordination has been used to trigger adaptive restructuring in modular robots~\cite{shen2002hormone,yang2017self}, while regenerative and developmental approaches have explored local message-passing, gradients, and self-organized recovery~\cite{rubenstein2009biologically,stoy2004self,thalamy2019survey}. More recently, Shreesha and Levin~\cite{shreesha2024stress} formalize \emph{stress sharing}: when units exchange information about their internal strain, the collective coordinates restructuring through a lightweight local channel with no central coordinator. These ideas are highly suggestive for robotics, but existing implementations do not explicitly enforce connectivity preservation and physically admissible motion during damage-responsive reconfiguration.

\subsection{Contributions}

Taken together, the literature reveals an important gap. Existing approaches either assume global planning and persistent connectivity, operate on abstract graphs without physical motion constraints, or provide bio-inspired coordination mechanisms without explicitly enforcing connectivity preservation during execution. To the best of our knowledge, there is no existing framework for \emph{damage-responsive repair} in modular robotic spacecraft that simultaneously satisfies the following requirements: strictly local decision-making, lattice-constrained physically admissible motion, and connectivity-safe execution throughout the repair process.

This paper contributes a fully decentralized, asynchronous framework that repairs a damaged modular spacecraft using only local sensing and communication, pairing a bio-inspired stress-sharing policy for connectivity restoration with a shape-recovering restructuring phase, together with a high-fidelity PyBullet rigid-body evaluation using physically realized rolling-sphere pivots on assemblies of up to $160$ modules at $10$--$30\%$ damage.

\section{Problem Formulation}
\label{sec:problem}

We model a modular spacecraft as a lattice-connected assembly of interchangeable modules, where we assume only the relative position of the modules is relevant and not their relative attitudes. At time $t$, the system is a directed graph $\mathcal{G}_t = (V, E_t, g_t)$, where $V$ is the set of modules (vertices; $u,v,w \in V$ denote arbitrary modules) and $E_t \subseteq V \times V$ the set of directed edges representing physical connections. The map $g_t : E_t \to \mathbb{R}^3$ assigns each directed edge $(u,v)$ its relative translation $g_t(u,v) = \boldsymbol{\rho}_{uv} \in \mathbb{R}^3$, with inverse $g_t(v,u) = -\boldsymbol{\rho}_{uv} = \boldsymbol{\rho}_{vu}$. Because all edge transformations are pure translations, composition along any path reduces to vector addition: for consecutive edges $(u,v),(v,w)\in E_t$, $g_t(u,w) = \boldsymbol{\rho}_{uv} + \boldsymbol{\rho}_{vw}$. Each relative translation is constrained to a unit lattice direction, $\boldsymbol{\rho}_{uv}\in\{\pm \hat{\boldsymbol{x}},\, \pm \hat{\boldsymbol{y}},\, \pm \hat{\boldsymbol{z}}\}$ with $\|\boldsymbol{\rho}_{uv}\| = 1$. For a path $u \rightsquigarrow w = (v_0, \dots, v_k)$ with $v_0=u$, $v_k=w$, and $(v_{i-1},v_i) \in E_t$, the displacement function $d_t(u,w) = \sum_{i=1}^{k} g_t(v_{i-1},v_i)$ returns the relative displacement between two vertices by composing the lattice translations along the path. We assume cycle consistency, $d_t(u,u) = \boldsymbol{0}$.

\subsection{Active Subgraph and Damage}

A subset $\bar{V}_t \subseteq V$ denotes the set of \emph{active} modules at time $t$, with active edge set $\bar{E}_t \subseteq \bar{V}_t \times \bar{V}_t$ and active subgraph $\bar{\mathcal{G}}_t = (\bar{V}_t, \bar{E}_t, \bar{g}_t)$, where $\bar{g}_t$ restricts the edge map to active edges. Only active modules participate in reconfiguration, and $\bar{u}, \bar{v}, \bar{w} \in \bar{V}_t$ denote arbitrary active modules. A damage event at time $t$ induces a new active set $\bar{V}_{t+1} = \bar{V}_t \setminus F_t$, where $F_t \subseteq \bar{V}_t$ is the set of damaged or failed modules at that time.
All physical connections remain present in the underlying graph $E_t$, but edges incident to inactive vertices are excluded from $\bar{E}_{t+1}$.
As a result, the active subgraph $\bar{\mathcal{G}}_{t+1}$ may become disconnected, necessitating reconfiguration to restore structural or functional connectivity. The neighbor set for a vertex $u$ at time $t$ is defined as $N_t(u) := \{\, w \mid (u,w) \in E_t \,\}$, and the \textit{active} neighbor set for a vertex $\bar{u}$ is defined as $\bar{N}_t(\bar{u}) := \{\, \bar{w} \mid (\bar{u},\bar{w}) \in \bar{E}_t \,\}$.

\subsection{Detachability and Attachability}
\label{sec:detachattach}

A vertex $\bar{u}$ may participate in reconfiguration subject to the
following rules.

\paragraph{Detachability}
A vertex $\bar{u}$ may detach from a neighboring vertex
$v$ by removing both directed edges $(\bar{u},v)$ and $(v,\bar{u})$ from $E_t$,
resulting in
\begin{equation}
\label{eq:detach}
E_{t+1}
\;=\;
E_t \setminus \{(\bar{u},v),(v,\bar{u})\},
\end{equation}
if and only if $|\bar{N}_{t+1}(\bar{u})| \ge 1$. That is, the vertex initiating detachment must remain connected to at
least one other active neighbor after the operation. The vertex $v$ may
be either active or inactive.

\paragraph{Attachability}
A vertex $\bar{u}$ may attach to a vertex $v$ by adding
directed edges $(\bar{u},v)$ and $(v,\bar{u})$ to the graph,
\begin{equation}
\label{eq:attach}
E_{t+1}
\;=\;
E_t \cup \{(\bar{u},v),(v,\bar{u})\},
\end{equation}
if and only if the following conditions hold:
(i) the relative displacement is a unit lattice direction, $\boldsymbol{\rho}_{\bar{u}v}\in\{\pm \hat{\boldsymbol{x}},\, \pm \hat{\boldsymbol{y}},\, \pm \hat{\boldsymbol{z}}\}$, and (ii) that direction is not already occupied by an existing edge incident to $v$, i.e.\ $\boldsymbol{\rho}_{\bar u v} \neq \boldsymbol{\rho}_{xv}$ for all $x \in N_t(v)$. When the attachment occurs, the edge transformations are $g_{t+1}(\bar{u},v) = \boldsymbol{\rho}_{\bar{u}v}$ and $g_{t+1}(v,\bar{u}) = -\boldsymbol{\rho}_{\bar{u}v}$.

\subsection{Shape Similarity}

We summarize the shape of the active subgraph at time $t$ by its
\emph{inter-module distance matrix} $P_t$, with entries $[P_t]_{\bar{u}\bar{v}} = \| d_t(\bar{u},\bar{v}) \|_2$ over $\bar{u},\bar{v} \in \bar{V}_t$, the pairwise distances between active modules.

For two configurations with inter-module distance matrices $P$ and $Q$,
their shape difference is the square-loss \emph{Gromov--Wasserstein} (GW)
discrepancy~\cite{memoli2011gromov}
\[
\mathrm{diff}(P,Q)
=
\min_{\pi \in \Pi(P,Q)}
\left(
\sum_{i,k,j,l}
\big|P_{ik}-Q_{jl}\big|^2\,
\pi_{ij}\pi_{kl}
\right)^{1/2},
\]
where $\Pi(P,Q)$ denotes the set of admissible correspondences between the
active vertex sets represented by $P$ and $Q$. Each correspondence is
represented by a nonnegative matrix $\pi$ whose row and column sums are
uniform over the active vertices of the respective configurations, with
$\pi_{ij}$ indicating how strongly active vertex $i$ in the first
configuration is matched to active vertex $j$ in the second. The matrices are normalized by their common maximum, so
$\mathrm{diff}(P,Q)\in[0,1]$.

Following a damage event at time $t$, the active subgraph
$\bar{\mathcal{G}}_{t+1}$ may be disconnected.
We seek a sequence of admissible reconfiguration operations that restores
active connectivity while minimizing deviation from the pre-damage
shape. Writing $P_t$ for the pre-damage shape and $P_{t_{\mathrm{f}}}$ for the terminal shape at time $t_{\mathrm{f}} \ge t+1$, the reconfiguration problem is to minimize $\mathrm{diff}(P_t, P_{t_{\mathrm{f}}})$ over admissible operation sequences $\{\bar{\mathcal{G}}_\tau\}_{\tau=t+1}^{t_{\mathrm{f}}}$
such that $\bar{\mathcal{G}}_{t_{\mathrm{f}}}$ is connected and every intermediate operation obeys the detachability and attachability rules~\eqref{eq:detach}--\eqref{eq:attach} of Section~\ref{sec:detachattach}.

\section{Proposed Approach: Distributed Stress-Sharing Repair}
\label{sec:approach}

\subsection{Failure Signals and Directional Propagation}

We propose an asynchronous distributed policy inspired by stress sharing: failure information diffuses through the active subgraph, biasing non-critical modules to migrate toward the fault while preserving connectivity. This coagulation process is detailed in Algorithm \ref{alg:coagulation}. The full proposed approach is summarized in Figure \ref{fig:repair_process}.

\begin{figure}[!t]
  \centering
  \includegraphics[width=\linewidth]{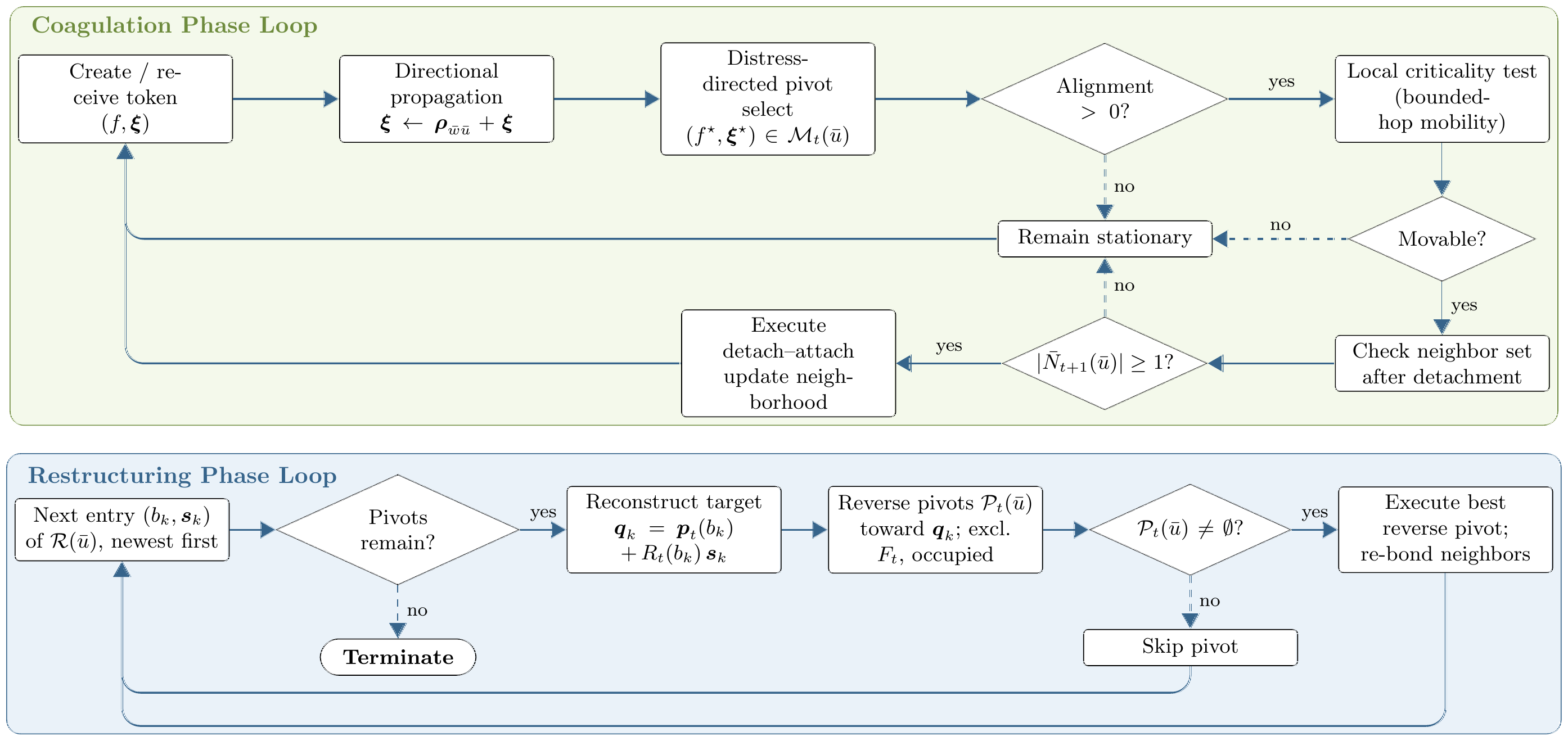}
  \caption{Per-agent control flow of the two-phase distributed repair process. \emph{Top (coagulation):} an agent propagates a distress token and, if a pivot improves alignment with the fault and the bounded-hop criticality test certifies it as movable, executes a connectivity-safe detach--attach toward the fault; otherwise it stays and forwards. \emph{Bottom (restructuring):} each displaced agent retraces its coagulation pivots in reverse, reconstructing each source cell relative to the current pose of its bonded neighbor and executing the connectivity-safe reverse pivot.}
  \label{fig:repair_process}
\end{figure}

Each ego agent $\bar{u}$ maintains a
set of \emph{distress tokens} $\mathcal{M}_t(\bar{u}) \subseteq F_t \times \mathbb{R}^3$, where a token $(f,\boldsymbol{\xi})$ indicates a suspected failed vertex
$f \in F_t$ and a direction $\boldsymbol{\xi}$ expressed in the frame of $\bar{u}$.

\paragraph{Local detection}
If $\bar{u}$ detects that a former neighbor $f$ is inactive (e.g., heartbeat
timeout), it generates the token $(f,\boldsymbol{\xi})$ with $\boldsymbol{\xi} := \boldsymbol{\rho}_{\bar{u}f}$, the last-known displacement prior to deactivation.

\paragraph{Propagation}
When $\bar{u}$ receives or creates a token $(f,\boldsymbol{\xi})$, it
updates and forwards it by composing the edge translation, $\boldsymbol{\xi} \leftarrow \boldsymbol{\rho}_{\bar{w}\bar{u}} + \boldsymbol{\xi}$, broadcasting $(f,\boldsymbol{\xi})$ to neighbors $\bar{w}\in\bar{N}_t(\bar{u})$.
Because path composition is additive under pure translations,
$\boldsymbol{\xi}$ approximates the displacement from neighbors $\bar{w}$ toward the fault location.

\subsection{Local Criticality Test}
\label{sec:crittest}

When $\mathcal{M}_t(\bar u)$ is nonempty, the ego agent $\bar u$ must decide
whether it can move without creating a new disconnect in the active graph
$\bar{\mathcal G}_t$.

We use the conservative local mobility test from prior network-repair
work~\cite{zhang2018autonomous}, parameterized by a safety radius
$r_{\mathrm{safe}}\ge 2$: the largest number of hops over which $\bar u$ verifies
that its neighbors would remain mutually connected after it leaves. Let
$\bar{\mathcal G}_t-\bar u$ be the active graph with $\bar u$ and its incident
edges deleted, and $d_{\bar{\mathcal G}_t-\bar u}(\bar x,\bar v)$ the
shortest-path distance between $\bar x$ and $\bar v$ in that graph ($\infty$ if
disconnected). The agent $\bar u$ is declared \emph{movable} if it is a leaf
($|\bar N_t(\bar u)|\le 1$) or every pair of distinct neighbors stays connected
within $r_{\mathrm{safe}}$ hops after its disconnection, i.e.\ $d_{\bar{\mathcal G}_t-\bar u}(\bar x,\bar v)\le r_{\mathrm{safe}}$ for all distinct $\bar x,\bar v\in \bar N_t(\bar u)$.
The test thus asks a purely local question: if $\bar u$ disappears, can all of
its former neighbors still reach one another by short paths? We use
$r_{\mathrm{safe}}=2$ throughout, the minimal admissible radius.

\begin{lemma}[Conservative sufficiency of the local mobility test]
\label{lem:two_hop_conservative}
Let $\bar{\mathcal G}_t=(\bar V_t,\bar E_t,\bar g_t)$ be the active graph. If
$\bar u$ passes the $r_{\mathrm{safe}}$ local mobility test, then disconnecting
$\bar u$ does not disconnect the active component containing $\bar u$.

The converse is false: deleting $\bar u$ can be globally safe even when the
bounded-hop test rejects it.
\end{lemma}

\begin{proof}
If $|\bar N_t(\bar u)|\le 1$, then $\bar u$ is a leaf or isolated vertex, and
removing it cannot disconnect any two other vertices. Now suppose
$|\bar N_t(\bar u)|>1$. Because $\bar u$ passes the test, every pair of distinct
neighbors $\bar x,\bar v\in \bar N_t(\bar u)$ is joined by a path
$\bar x \rightsquigarrow \bar v \subseteq \bar{\mathcal G}_t-\bar u$ of length at
most $r_{\mathrm{safe}}$, so all former neighbors of $\bar u$ lie in one
connected component of $\bar{\mathcal G}_t-\bar u$.

For non-necessity, take a simple even cycle of length $\ell$ and any vertex
$\bar u$ on it. Deleting $\bar u$ leaves a replacement path of length $\ell-2$
between its two former neighbors, so the deletion is globally safe; yet the
bounded-hop test accepts it only if $\ell-2\le r_{\mathrm{safe}}$, i.e.\
$\ell\le r_{\mathrm{safe}}+2$. Thus any even cycle with
$\ell > r_{\mathrm{safe}}+2$ is globally safe but locally rejected (the first
rejected cycle is the six-cycle for $r_{\mathrm{safe}}=2$ and the octagon for
$r_{\mathrm{safe}}=4$).
\end{proof}

\begin{remark}[Conservatism]
\label{rem}
Increasing $r_{\mathrm{safe}}$ admits more moves at the cost of a larger local search; eliminating the conservatism entirely would require global articulation-point information, violating the decentralized assumption of this work.
\end{remark}

\subsection{Distress-Directed Pivot Selection}
\label{sec:pivotselection}
Once an agent $\bar{u}$ is certified movable, it selects the closest target distress token $(f^\star, \boldsymbol{\xi}^\star) = \argmin_{(f,\boldsymbol{\xi})\in\mathcal{M}_t(\bar{u})} \|\boldsymbol{\xi}\|_2$. The agent evaluates all physically admissible pivot motions, meaning those satisfying the detachability and attachability constraints in Section~\ref{sec:detachattach}, where each candidate motion induces a local lattice displacement $\Delta \boldsymbol{p} \in \mathbb{R}^3$. Let $\langle \cdot\, ,\cdot\rangle$ denote the inner product of two vectors in $\mathbb{R}^3$. To prioritize movement toward the fault, the agent selects the pivot that maximizes alignment with the target direction:
\begin{equation}
\Delta \boldsymbol{p}^\star = \arg\max_{\Delta \boldsymbol{p}} \langle \Delta \boldsymbol{p}, \boldsymbol{\xi}^\star \rangle,
\label{eq:pivot_alignment}
\end{equation}
subject to the \emph{alignment constraint} $\langle \Delta \boldsymbol{p}^\star, \boldsymbol{\xi}^\star \rangle > 0$.

\paragraph{Anti-oscillation memory and stochastic exploration}
Each agent keeps a small local set $\mathcal{H}_t(\bar{u})$ of previously visited module pairs and rejects any candidate pivot already in it, breaking oscillations on symmetric structures where two cells repeatedly tie under~\eqref{eq:pivot_alignment}; once every admissible pair is exhausted the agent remains stationary, giving an agent-level termination certificate. When no admissible pivot strictly improves alignment, the agent explores: with probability $\varepsilon$ (set to $1.0$) it takes an admissible pivot chosen uniformly at random rather than staying put. This is the primary deadlock-breaking mechanism, letting agents escape local configurations while $\mathcal{H}_t(\bar{u})$ prevents revisits; every exploratory pivot still passes the criticality test and respects the per-agent motion budget, so exploration can neither break connectivity nor run unbounded.

\begin{algorithm}[t]
\caption{Coagulation Policy}
\label{alg:coagulation}
\begin{algorithmic}[1]
\STATE \textbf{Input:} agent $\bar u$; tokens $\mathcal{M}_t(\bar u)$; budget $a_{\mathrm{move}}(\bar u),\,a_{\mathrm{fwd}}(\bar u)$
\STATE \textbf{while} $a_{\mathrm{move}}(\bar u)$ \textbf{and} $\,a_{\mathrm{fwd}}(\bar u) > 0$:
\STATE \hspace{1em} \textbf{for} $f \in N_t(\bar u)$ \textbf{with} $f \notin \bar{V}_t$ \textbf{and} $a_{\mathrm{fwd}}(\bar u) > 0$:
\STATE \hspace{2em} $\mathcal{M}_t(\bar u) \gets \mathcal{M}_t(\bar u) \cup \{(f,\,\boldsymbol{\rho}_{\bar u f})\}$
\STATE \hspace{1em} \textbf{if} $\mathcal{M}_t(\bar u) = \emptyset$: \STATE \hspace{2em} \textbf{continue}
\STATE \hspace{1em} $\displaystyle (f^\star,\boldsymbol{\xi}^\star) \gets \argmin_{(f,\boldsymbol{\xi})\,\in\,\mathcal{M}_t(\bar u)} \lVert\boldsymbol{\xi}\rVert$ 
\STATE \hspace{1em} \textbf{if} $a_{\mathrm{move}}(\bar u) > 0$ \textbf{and} Criticality Test (Sec.~\ref{sec:crittest}):
\STATE \hspace{2em} $\displaystyle
\Delta\boldsymbol{p}^\star \gets
\argmax_{\substack{
\Delta\boldsymbol{p}\notin\mathcal{H}_t(\bar u)
}}
\left\langle \Delta\boldsymbol{p},\,\boldsymbol{\xi}^\star \right\rangle$
\STATE \hspace{2em} \textbf{if} $\langle \Delta\boldsymbol{p}^\star,\boldsymbol{\xi}^\star\rangle > 0$ \textbf{or} $\mathrm{rand}() < \varepsilon$:
\STATE \hspace{3em} execute $\Delta\boldsymbol{p}^\star$ about bonded neighbor $b$;
\STATE \hspace{3em} $a_{\mathrm{move}}(\bar u)\!\mathrel{-}=\!1$;
\STATE \hspace{3em} $\mathcal{H}_{t+1}(\bar u)\gets\mathcal{H}_t(\bar u)\cup\{\Delta\boldsymbol{p}^\star\}$;
\STATE \hspace{3em} $\mathcal{R}(\bar u).\mathrm{append}(b,\boldsymbol{s})$; (Sec. \ref{sec:reformation})
\STATE \hspace{3em} \textbf{continue}
\STATE \hspace{1em} \textbf{for} $\bar w \in \bar{N}_t(\bar u)$: 
\STATE \hspace{2em} $\mathcal{M}_{t+1}(\bar w) \gets \mathcal{M}_t(\bar w) \cup \{(f^\star,\,\boldsymbol{\rho}_{\bar w\bar u}+\boldsymbol{\xi}^\star)\}$ 
\STATE \hspace{1em} $a_{\mathrm{fwd}}(\bar u)\!\mathrel{-}=\!1$
\end{algorithmic}
\end{algorithm}

\subsection{Action Budgets and Message Complexity}
\label{sec:token_lifecycle}

\paragraph{Action budgets}
Each agent begins the phase with two integer budgets: a motion budget $a_{\mathrm{move}}$ and a communication budget $a_{\mathrm{fwd}}$. A movable token-holder spends one $a_{\mathrm{move}}$ per executed pivot; if it cannot or does not move, it instead relays the token, spending one $a_{\mathrm{fwd}}$, and a fault-adjacent source likewise stops seeding once its $a_{\mathrm{fwd}}$ is spent. Agents pivot concurrently under a bounded-hop exclusion lock ($r_{\mathrm{exc}}=4$) with a fresh random visit order each tick, assuming no global clock or priority scheduling. A phase ends at \emph{quiescence}, when every agent is budget-exhausted or idle and none is pivoting; the finite budgets ($a_{\mathrm{move}}=5$, $a_{\mathrm{fwd}}=50$) bound moves and forwards, making this a genuine fixed point rather than a global termination condition.

\paragraph{Message complexity}
A forwarding wave at $\bar{u}$ emits at most $|\bar{N}_t(\bar{u})|$ messages, so a tick in which every holder forwards costs $O(|\bar{E}_t|)$ token-creations, and the per-agent communication budget caps the cumulative count over a phase at $O(|\bar V_t|\,a_{\mathrm{fwd}})$. Total messages to convergence are problem-dependent, governed by the number of pivots and the wavefront geometry, so Section~\ref{sec:simulation} reports move counts as the operative measure of repair effort.

\subsection{Restructuring After Coagulation}
\label{sec:reformation}

After coagulation the structure is reconnected but its shape is distorted, since agents migrated away from their pre-damage positions. As an optional second phase, each displaced agent attempts to recover shape by retracing its own coagulation pivots in reverse (Fig.~\ref{fig:repair_process}, right). To this end each agent records the ordered sequence $\mathcal{R}(\bar u)=((b_1,\boldsymbol{s}_1),\dots)$ of pivots it executes, where $b_k$ is the bonded neighbor rotated about and $\boldsymbol{s}_k$ the pre-pivot position in $b_k$'s body frame. Replaying $\mathcal{R}(\bar u)$ newest-first, the target for step $k$ is the source cell reconstructed against the neighbor's current pose, $\boldsymbol{q}_k=\boldsymbol{p}_t(b_k)+R_t(b_k)\,\boldsymbol{s}_k$; the agent takes the criticality-gated admissible pivot that most reduces $\|\boldsymbol{p}_t(\bar u)-\boldsymbol{q}_k\|$ (the exact inverse pivot when $b_k$ is still bonded), skipping steps whose neighbor is inactive. As the results below show, this recovers only modest shape, so we treat restructuring as a lightweight add-on and leave stronger recovery to future work.

\begin{table}[t]
\centering
\caption{Results of Monte Carlo Trials in Simulated Damage Scenarios}
\label{tab:results_summary_stacked}
\setlength{\tabcolsep}{3.5pt}
\renewcommand{\arraystretch}{1.1}

\begin{tabular}{p{1.55cm} c|ccc|ccc}
\hline
 & & \multicolumn{3}{c|}{\textbf{Tree (\% damaged)}} & \multicolumn{3}{c}{\textbf{FC (\% damaged)}} \\
\textbf{Metric} & $\boldsymbol{n}$ & \textbf{10\%} & \textbf{20\%} & \textbf{30\%} & \textbf{10\%} & \textbf{20\%} & \textbf{30\%} \\
\hline

\multirow{3}{=}{Full\\Reconnection\\Rate (\%)}
& 10  & 87 & 71 & 45 & 85 & 68 & 51 \\
& 80  & 56 & 30 & 6  & 45 & 18 & 3 \\
& 160 & 37 & 10 & 1  & 25 & 3  & 0 \\

\hline
\multirow{3}{=}{Restoration\\(\%)}
& 10  & 96 & 92 & 85 & 96 & 92 & 85 \\
& 80  & 98 & 96 & 91 & 97 & 93 & 84 \\
& 160 & 98 & 97 & 93 & 98 & 95 & 88 \\

\hline
\multirow{3}{=}{Shape diff.\\(\%)}
& 10  & 18 & 19 & 19 & 18 & 19 & 19 \\
& 80  & 5  & 7  & 8  & 6  & 8  & 7 \\
& 160 & 5  & 7  & 7  & 5  & 6  & -- \\

\hline
\multirow{3}{=}{Phase 1\\Moves}
& 10  & 6   & 7   & 6   & 6   & 7   & 7 \\
& 80  & 75  & 69  & 47  & 85  & 73  & 49 \\
& 160 & 201 & 167 & 99  & 210 & 154 & 94 \\
\hline
\end{tabular}
\end{table}

\begin{figure}[!t]
  \centering
  \includegraphics[width=\linewidth]{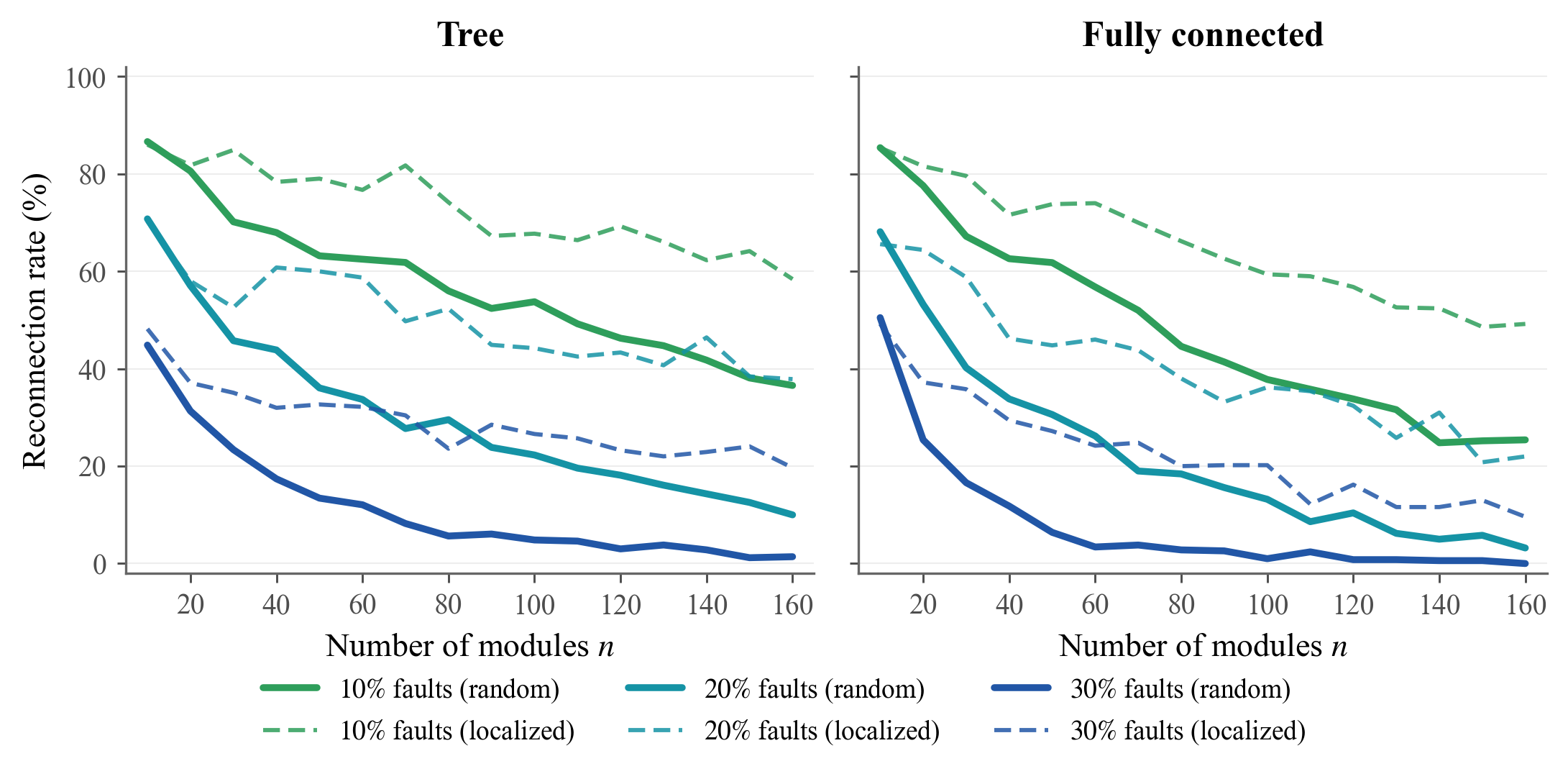}\\[4pt]
  \includegraphics[width=\linewidth]{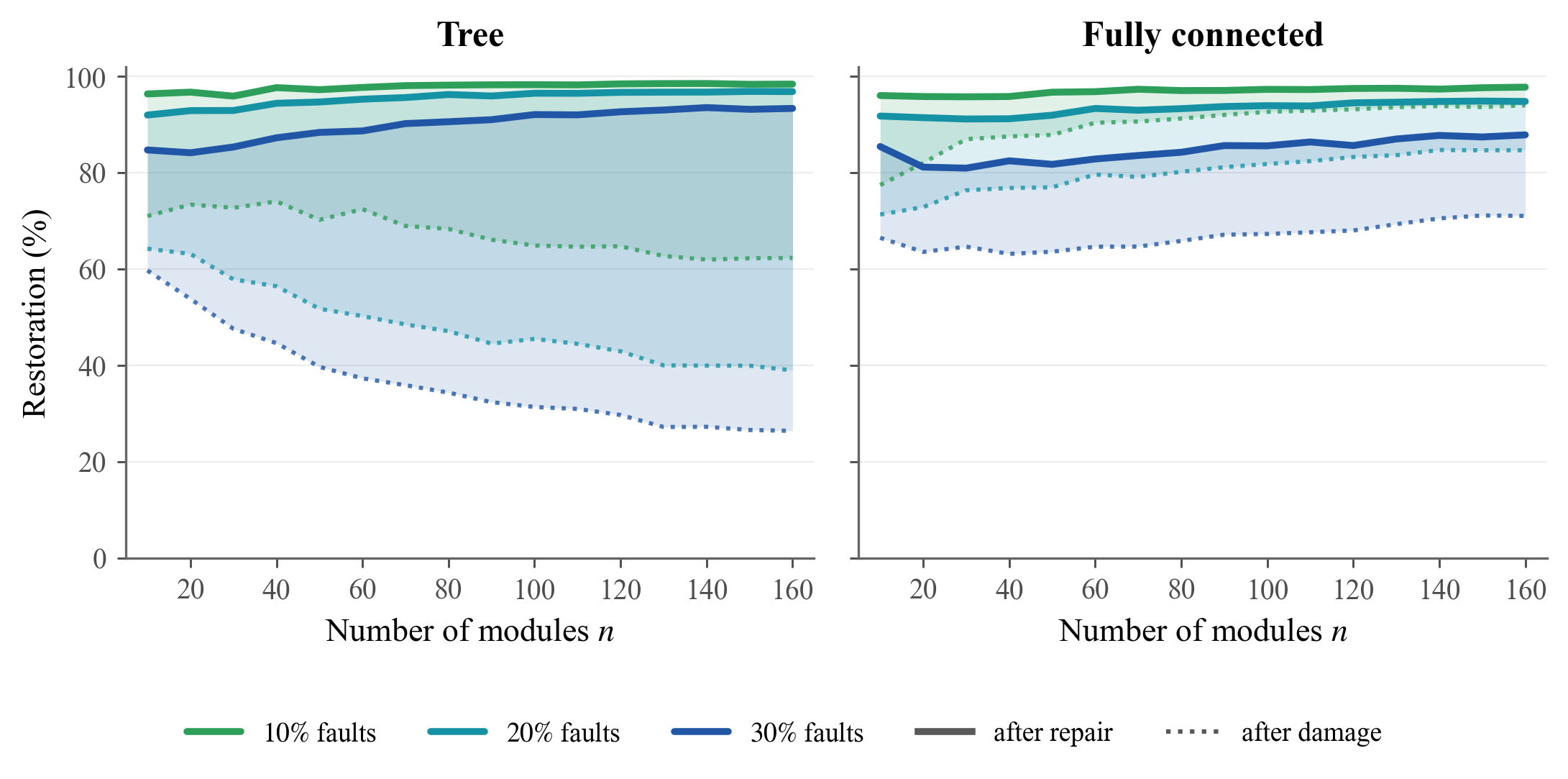}
  \caption{Connectivity restoration versus number of modules $n$. Top: full reconnection rate for tree (left) and fully connected (right) structures; solid random, dashed localized faults, colored by density; it falls with density and scale. Bottom: restoration under random faults, dotted immediately after damage and solid after phase~1, the gap being the policy's contribution, which stays $\gtrsim\!80\%$ and rises with $n$.}
  \label{fig:reconnection_rate_overlay}
  \label{fig:restoration_fraction}
\end{figure}

\section{Simulations}
\label{sec:simulation}

We evaluate the coagulation--restructuring policy by Monte Carlo on \emph{tree} (acyclic) and \emph{fully connected} (FC, dense multi-neighbor) lattice assemblies of $n\in[10,160]$ modules. Both topologies are grown by a random walk on the cubic lattice, an FC structure bonding each new module to all its placed neighbors and a tree to a single parent. A damage event is a \emph{random} (uniform) or \emph{localized} (contiguous, breadth-first-grown) fault set, drawn by rejection sampling until it disconnects the active graph; we sweep three densities ($10/20/30\%$) at $500$ trials per cell. A parameter sweep sets the operating point: reconnection is insensitive to the safety radius over $r_{\mathrm{safe}}\in[2,4]$ (we take the cheapest, $r_{\mathrm{safe}}=2$); exploration aids dense FC structures (reconnection rising from $\approx\!18\%$ at $\varepsilon=0$ to $\approx\!34\%$ at $\varepsilon=1$) and is neutral on trees (we use $\varepsilon=1.0$); and denser token emission improves reconnection and is affordable under the decoupled communication budget (we use $\tau_{\mathrm{gen}}=2$).

We report four quantities. \emph{Full reconnection rate} is the fraction of trials whose final active subgraph is a single connected component. \emph{Restoration} is $L/A$, the fraction of the $A$ active (non-fault) modules lying in the largest active component ($L$) at phase-1 end: it is $100\%$ when the survivors form one body, and excluding faults from both terms makes it invariant to how many modules were lost. We also report the shape difference $\mathrm{diff}(P_t,P_{t_{\mathrm f}})$ between the pre-damage and final post-restructuring configurations, and mean phase-1 move counts. Only trials whose fault set actually disconnects the post-damage subgraph are counted as meaningful.

\subsection{Physical Realization}

We evaluate in a custom PyBullet rigid-body simulator: the decision-level policy of Section~\ref{sec:approach} runs exactly as specified, but as a contact-rich rigid-body world rather than kinematic placement. A $90^\circ$ pivot is a \emph{rolling-sphere coupling} between an agent and its bonded neighbor (one sphere rolling without slipping around another), settling into bonding tolerance. Being force-driven, each pivot incurs residual pose error, occasional non-convergence, and attitude drift. Clearance is certified at the lattice level, by a target-cell occupancy check plus the anti-oscillation memory, never by swept-volume queries. On completing a pivot an agent re-bonds to its new lattice neighbors within tolerance, and a non-converging pivot is reversed or snapped to the nearest cell, so only completed, re-bonded pivots count as moves.

\paragraph{Module geometry and the locality assumption}
Each module is a unit-diameter sphere with cardinal contact connectors, as in M-Blocks~\cite{romanishin2015mblocks}. This is load-bearing for locality: the criticality test (Section~\ref{sec:crittest}) and pivot selection (Section~\ref{sec:pivotselection}) consult only a bounded neighborhood, which certifies a safe pivot only because a sphere rotating $90^\circ$ about a contact neighbor sweeps exactly the source and target lattice cells; rotational symmetry thus collapses the clearance check to a lattice-cell occupancy check among graph-local modules. Non-spherical modules break this (a cube sweeps corners through graph-distant cells), so certifying clearance would need geometric awareness beyond the local radius. We therefore treat cubes, polyhedral voxels, and articulated modules as future work.

\subsection{Results}
\label{sec:results}

\paragraph{Two regimes of connectivity restoration}
The central finding is that \emph{full reconnection} and \emph{restoration} behave oppositely with scale (Table~\ref{tab:results_summary_stacked}). Full reconnection, the all-or-nothing measure of complete reunification, falls steeply with both density and $n$, and less so on trees than on FC (Fig.~\ref{fig:reconnection_rate_overlay}), reaching $\approx\!0\%$ for $30\%$ random damage at $n=160$. Restoration instead stays high, and is best read as the policy's \emph{improvement} over the post-damage state (Fig.~\ref{fig:restoration_fraction}): on trees, where damage shatters the structure, the gain is large and grows with severity (at $30\%$, $n=160$ the largest surviving component rises from $\approx\!26\%$ to $\approx\!93\%$ of survivors, a $+67$-point gain); dense FC structures fragment less, so the policy's contribution is smaller but the achieved level is comparable. Either way restoration stays $\gtrsim\!80\%$ and rises with $n$, so even a ``$2\%$-reconnection'' trial has gathered the large majority of survivors into one body. The residual fragments are a limit of strictly-local greedy repair: the last fragments are too far apart for any local pivot to bridge. Phase-1 effort grows with $n$ and density to a few hundred moves at $n=160$.

\paragraph{Spatial distribution dominates fault amount}
At equal density, the spatial pattern matters more than the amount. \emph{Localized} faults are far easier than \emph{random} (at $n=160$, tree $20\%$: $38\%$ vs.\ $10\%$; $30\%$: $20\%$ vs.\ $\approx\!1\%$; clustered falls between), because localized damage leaves one large component plus a compact hole the wavefront bridges directly, whereas random damage shatters the structure into many fragments that must each be reattached.

\paragraph{The stress-sharing gradient is a near-free gain}
No prior method targets this exact setting (strictly local, lattice-constrained, connectivity-safe repair), so there is no baseline in the literature we can readily compare against; we therefore evaluate our method against an ablation of its own directional bias. Selecting pivots uniformly at random instead of by fault alignment~\eqref{eq:pivot_alignment} while keeping the rest of the pipeline intact costs $\approx\!5$ points of reconnection on both topologies (tree $38\%\!\to\!33\%$, FC $33\%\!\to\!27\%$) at essentially no reduction in computation.

\paragraph{Shape recovery is limited}
Restructuring recovers only modest shape (a mean improvement of $\approx\!0.8$ points, and not monotonic), because the same criticality test that keeps the structure connected forbids any displaced module that has become an articulation point from pivoting home, so only a non-critical minority can move. Meaningful shape recovery would require relaxing the strict connectivity guarantee, e.g.\ permitting bounded temporary disconnections; we therefore leave it to future work; the shape values in Table~\ref{tab:results_summary_stacked} are post-restructuring, reported mainly to document this limit.

\section{Conclusions}
We presented a decentralized stress-sharing strategy that lets a damaged modular spacecraft repair itself using only local information and connectivity-safe pivots, validated under rigid-body physics. Three findings stand out. Consolidation, gathering the survivors into a single body, stays high and improves with assembly size, even where full reconnection becomes rare. The spatial concentration of damage, not its magnitude, is the primary determinant of repairability. And the residual failures at high density are a limit of the strictly-local greedy nature of our approach. The locality of the method rests on the rotational symmetry of spherical modules, whose pivot sweep stays within the source and target cells; other geometries would require longer-range clearance certification.

This limit points directly at the next step: a lightweight longer-range rendezvous mechanism, invoked only for the few fragments local repair cannot close, could convert high consolidation into full reconnection. Non-spherical and articulated modules, and convergence and shape-recovery guarantees, remain open for future work as well. More broadly, these results suggest that resilience in space robotics may rely less on redundancy and more on distributed reorganization inspired by biological repair, letting modular spacecraft maintain function after unexpected failures without external intervention.

\bibliographystyle{IEEEtran}
\bibliography{references}

\end{document}